\documentclass[lettersize,journal]{IEEEtran}
\usepackage{amsthm,amsmath,amssymb}
\usepackage{algorithmic}
\usepackage{algorithm}
\usepackage{array}
\usepackage[caption=false,font=normalsize,labelfont=sf,textfont=sf]{subfig}
\usepackage{textcomp}
\usepackage{stfloats}
\usepackage{url}
\usepackage{verbatim}
\usepackage{graphicx}
\usepackage{cite}
\usepackage{multirow}
\usepackage{multicol}
\usepackage{parallel}
\usepackage{booktabs}
\usepackage{threeparttable}
\usepackage{pifont} 
\usepackage{siunitx} 
\usepackage[figuresright]{rotating} 
\newtheorem{prop}{Proposition}
\newtheorem{theorem}{Theorem}
\newtheorem{lemma}{Lemma}
\newtheorem{remark}{Remark}

\usepackage{xcolor}
\newcommand{\tx}{\mathrm}

\begin{document}

\title{Efficient collision avoidance for autonomous vehicles in polygonal domains}

\author{Jiayu Fan,
        Nikolce Murgovski,
        Jun Liang
\thanks{This work is supported by the National Key Research and Development Program of China under Grant 2019YFB1600500 and the China Scholarship Council under Grant 202206320304.}
\thanks{Jiayu Fan and Jun Liang are with the College of Control Science and Engineering, Zhejiang University, Hangzhou 310027, China (e-mail: jiayu.fanhsz@gmail.com; jliang@zju.edu.cn).}
\thanks{Nikolce Murgovski is with the Department of Electrical Engineering, Chalmers University of Technology, 41296 G\"oteborg, Sweden (e-mail: nikolce.murgovski@chalmers.se).}
\thanks{*Corresponding author (e-mail:jliang@zju.edu.cn).}%
}

\maketitle

\begin{abstract}
This research focuses on trajectory planning problems for autonomous vehicles utilizing numerical optimal control techniques. The study reformulates the constrained optimization problem into a nonlinear programming problem, incorporating explicit collision avoidance constraints. We present three novel, exact formulations to describe collision constraints. The first formulation is derived from a proposition concerning the separation of a point and a convex set. We prove the separating proposition through De Morgan's laws. Then, leveraging the hyperplane separation theorem we propose two efficient reformulations. Compared with the existing dual formulations and the first formulation, they significantly reduce the number of auxiliary variables to be optimized and inequality constraints within the nonlinear programming problem. Finally, the efficacy of the proposed formulations is demonstrated in the context of typical autonomous parking scenarios compared with state of the art. For generality, we design three initial guesses to assess the computational effort required for convergence to solutions when using the different collision formulations. The results illustrate that the scheme employing De Morgan’s laws performs equally well with those utilizing dual formulations, while the other two schemes based on hyperplane separation theorem exhibit the added benefit of requiring lower computational resources.

\end{abstract}

\begin{IEEEkeywords}
Autonomous parking, efficient collision avoidance, optimal control, De Morgan’s laws, hyperplane separation theorem.
\end{IEEEkeywords}
\section{Introduction}\label{sec1}
\IEEEPARstart{C}{ollision} avoidance is a critical aspect of optimization-based trajectory planning for a wide range of autonomous systems, including ground vehicles, robots, and so on \cite{ref1, ref2, ref3, ref4, ref60}. This paper explores solution techniques that rely on formulating an optimal control problem (OCP), followed by discretization and iterative methods using a nonlinear programming (NLP) solver. Central to this approach is the construction of collision avoidance constraints in an explicit and differentiable form. However, collision avoidance constraints in polygonal domains pose significant challenges due to their nonconvex nature and non-smoothness \cite{ref5, ref6}. They need to be formulated by exploiting appropriate mathematical representations of shapes of vehicles and obstacles. Prevalent methods typically resort to approximating vehicle and obstacle boundaries with smooth functions  \cite{ref24, ref25} or simplistic shapes like ellipsoids and spheres \cite{ref7, ref8, ref42}. While these approximations offer computational advantages, they often lead to conservative solutions and can even hinder autonomous vehicles from finding collision-free trajectories in confined environments. To enable autonomous vehicles to identify optimal trajectories within limited maneuverable space, it is essential to accurately model the shapes of vehicles and obstacles. Additionally, the exact formulations of collision avoidance constraints should also be investigated.

In most applications, obstacles naturally take the form of nonconvex polygons, while vehicles themselves can either be convex or nonconvex polygons \cite{ref1, ref2}. Due to the fact that convexity is well studied and understood, nonconvex sets are often decomposed as the union of convex sets \cite{ref9, ref10}. However, when computational performance is in question, it is not trivial how to decompose in an efficient and practical way. The interested readers can refer to \cite{ref11} to choose appropriate decomposition methods. For exact collision constraints between two convex sets, several typical formulations have been developed. A natural choice is the disjunctive programming methods, and the problem can be formulated as a mixed-integer optimization problem through binary variables \cite{ref12, ref13}. However, this method is computationally expensive if solving a large number of integer variables \cite{ref13, ref26, ref27}. Various measures have been proposed to mitigate this issue. The authors in \cite{ref12} propose to use hyperplane arrangements associated with binary variables for constraint description in a multi-obstacle environment. An over-approximation region is characterized to reduce to strictly binary formulations at the price of being conservative. In implementing a combined mixed integer and predictive control formulation, \cite{ref14} introduces two ways of reducing the number of binary variables possibly accelerating the online computation.

Additionally, there are many researchers having an interest in computing the shortest distance between a pair of convex sets \cite{ref15, ref16}. The idea is that imposing collision avoidance constraints is identical to requiring nonnegative distance between the sets. A reliable algorithm for obtaining the Euclidean distance is given by \cite{ref15}, where the authors define polytopes by their vertices and introduce a decent procedure beneficial for the efficiency of the algorithm. Furthermore, the authors in \cite{ref16} designs a faster distance sub-algorithm that guides the algorithm by \cite{ref15} toward a shorter search path in less computing time. In \cite{ref17}, the authors focus on collision detection with the distance with Minkowski sum structure, which proves to be efficiently solved by hybrid gradient methods. Furthermore, the notion of signed distance between two convex sets is brought a widespread attention in, e.g., \cite{ref1, ref18, ref19}. The purpose of investigating the signed distance is to find a minimum-penetration trajectory when a collision cannot be avoided. The common technique is to soften the collision constraints by replacing the positive safety margin with negative slack variables. In general, it is also desirable to include slack variables to ensure the feasibility of nonconvex optimization problems. The signed distance is defined in \cite{ref1} and used also in \cite{ref10, ref18}. 

Recently, more and more researchers go back to the early research on reformulations of distance and signed distance between two convex sets with duality techniques \cite{ref2, ref10, ref17}. The authors in \cite{ref19} summarize methods on how to separate a point and a convex set and how to separate convex sets by Lagrange dual functions. In \cite{ref17}, the authors propose a new Minimum Norm Duality theorem, i.e., exploring a maximal distance between a pair of parallel hyperplanes that separates two sets. Through the strong duality of convex optimization, the collision constraints can be exactly reformulated as smooth and differentiable constraints. This approach has been applied to various kinds of applications spanning from robot avoidance maneuver \cite{ref20}, vessels docking \cite{ref21}, to autonomous car parking \cite{ref10}, et al. What's more, \cite{ref21} introduces two indicator collision avoidance constraints by exploiting the Farkas’ Lemma and culling procedures are proposed to reduce problem size by identifying and eliminating additional decision variables associated with edges of the polygonal obstacles.

Specifically, checking the vertices of convex sets is a practical way to formulate collision avoidance constraints. The idea is that a pair of convex sets are disjoint equals that their vertices are kept outside each other. The authors in \cite{ref22} impose collision constraints on vertices of the vehicle through triangle area-based criterion, and then keep all vertices outside the obstacle. In the collision part of \cite{ref23}, a ray method that checks whether a vertex is inside a polygon is used to avoid obstacles for a tractor-trailer system.

In this research, we focus on optimization-based trajectory planning problems aimed at achieving effective collision avoidance in vehicles and obstacles with convex or nonconvex polygonal shapes. We give a proposition on separating a point from a convex set, and prove it with De Morgan's laws, which has not been used in previous work on parking scenarios. 
Based on the proposition and hyperplane separation theorem, we propose three explicit and exact formulations of collision avoidance constraints. Without the need of approximating the obstacle and vehicle geometry, the proposed formulations can transform the constrained optimal control planning problem to a smooth NLP problem, which can be solved by off-the-shelf solvers with a gradient-based algorithm. The key contributions of the work can be summarized as follows:
\begin{enumerate}
\item{We propose a smooth method to separate a point from a convex set, and introduce a novel proving technique with De Morgan’s laws.}
\item{We present three novel, exact formulations of collision constraints, which are highly beneficial for autonomous vehicles operating within a confined area. These formulations enable the identification of optimal trajectories that may not be achievable using approximate formulations.}
\item{We introduce two efficient collision formulations based on hyperplane separation theorem. Comparative analysis with state-of-the-art methods demonstrates a substantial reduction in the NLP problem size, resulting in improved computational efficiency and faster trajectory planning for autonomous vehicles.}
\end{enumerate}

The remainder of this paper is organized as follows. Section \ref{sec2} states a trajectory planning problem involving collision avoidance and formulates the planning problem as a unified OCP problem. Section \ref{sec3} declares a proposition on how to separate a point and a convex polygon. In Section \ref{sec4}, vehicles, obstacles, and driving environments are described by mathematical representations, and then three exact collision constraints are reformulated through the mentioned proposition and hyperplane separation theorem. Section \ref{sec5} analyzes the problem size compared to the state-of-the-art formulations of collision constraints. The constrained OCP is reformulated as a smooth NLP problem in Section \ref{sec6}. Moreover, Section \ref{sec7} demonstrates the efficacy of the proposed methods compared with existing methods. Finally, the concluding remarks and plans are provided in Section \ref{sec8}.

\section{Problem statement}\label{sec2}
In this paper, trajectory planning problems are studied for an autonomous vehicle that avoids collision with obstacles.  Vehicle dynamics are modeled as
\begin{equation}
\label{eq1}
\dot{\boldsymbol{\xi}} = f\left(\boldsymbol{\xi},\boldsymbol{u}\right), \ \boldsymbol{\xi}(0)=\boldsymbol{\xi}_{\tx{init}},
\end{equation}
where $\boldsymbol{\xi}$ is the state vector, $\boldsymbol{\xi}_{\tx{init}}$ is the initial state, $\boldsymbol{u}$ is the control input and $f$ is the system dynamics.

The driving environment (the canvas, or background) is modeled as a convex polygon $\mathcal W$, and the vehicle and obstacles are represented by a union of polygons, denoted as $\mathcal B$ and $\mathcal O$, respectively. The constrained OCP under study is minimizing a performance measure, planning the motion of the vehicle, represented by the set $\mathcal B$, from its initial state $\boldsymbol{\xi_{\tx{init}}}$ to a terminal state $\boldsymbol{\xi_{\tx{final}}}$ while always residing within the polygon $\mathcal W$, and not colliding with the obstacles $\mathcal O$. The problem is formulated as 
\begin{subequations}\label{eq2}
\begin{align}
\underset {\boldsymbol{\xi}, \boldsymbol{u}}{\text{min}} \ & \int_0^{t_\tx{f}}  
 \ell\left(\boldsymbol{\xi}(t), \boldsymbol{u}(t)\right)\tx{d}t  \label{eq2a}\\
\text {s.t.}  \ & \boldsymbol{\xi}(0)=\boldsymbol{\xi_{\tx{init}}},\ \boldsymbol{\xi}(t_\tx{f})=\boldsymbol{\xi_{\tx{final}}} \label{eq2b}\\
 & \dot{\boldsymbol{\xi}}(t) = f\left(\boldsymbol{\xi}(t),\boldsymbol{u}(t)\right), \label{eq2c}\\
 & \boldsymbol{\xi}(t) \in \mathcal{X},\ \boldsymbol{u}(t) \in \mathcal{U}, \label{eq2d}\\
 & \mathcal B(\boldsymbol{\xi}) \subset \mathcal W, \label{eq2e} \\ 
 & \mathcal B(\boldsymbol{\xi}) \cap \mathcal O=\emptyset, \label{eq2f}
\end{align}
\end{subequations}
where $t_\tx{f}$ is the final time, and $\mathcal{X}$ and $\mathcal{U}$ are the admissible sets of state and control, respectively. Constraints (\ref{eq2c})-(\ref{eq2f}) are imposed for $\forall t\in \left[0, t_\tx{f}\right]$. For the sake of clarity, the time dependence of $\mathcal O$ is omitted, but we remark that the proposed method can directly be applied to problems with moving obstacles. The function $ \ell$ represents the stage cost. It can be noticed that if the problem (\ref{eq2}) is to be solved by generic solvers using gradient-based algorithms, collision constraints (\ref{eq2e}) and (\ref{eq2f}) should be transformed into smooth and differentiable constraints in an explicit form. In the following, we first outline general propositions about separating polygons and then reformulate  (\ref{eq2e}) and (\ref{eq2f}) as explicit constraints by exploiting the propositions.

\section{Preliminaries}\label{sec3}
A convex polygon $\mathcal C$ can be represented with a linear matrix inequality as 
\begin{equation}\label{eq3}
\mathcal C :=\left\{\boldsymbol{q} \in \mathbb{R}^2 \mid A \boldsymbol{q} \leq \boldsymbol{b}\right\},
\end{equation}
where $A\in\mathbb R^{{N}\times 2}$, $\boldsymbol{b}\in\mathbb R^{{N}}$, and $\boldsymbol{q}$ is the position of a point q. Here the capital letters represent matrices, while the bold letters represent vectors. Let $\boldsymbol{a_i}$ be the {\it i}-th row vector in $A$ and let $b_i$ be the {\it i}-th entry in $\boldsymbol{b}$. Next, we give general propositions on how to separate a point from a convex polygon. 

\begin{lemma}[De Morgan's law]\label{lemma1}
    The negation of a conjunction is equivalent to the disjunction of the negations \cite{ref28}.
\end{lemma}
 
\begin{prop}\label{prop1}
    Let a point $q$ and a convex polygon $\mathcal C$ be represented by (\ref{eq3}). Then 
$\tx{q} \cap \mathcal C=\emptyset, \iff \exists \boldsymbol{\lambda} \geq \boldsymbol{0}: (A \boldsymbol{q}- \boldsymbol{b})^{\top} \boldsymbol{\lambda} >0.$
\end{prop}

\begin{proof}
The expression $\tx{q} \cap  \mathcal C \ne \emptyset$  states that q will be inside the polygon $\mathcal C$, which is identical to stating all inequalities
$\boldsymbol{a_1} \boldsymbol{q}- b_1 \le 0, \ \boldsymbol{a_2} \boldsymbol{q}- b_2 \le 0, \cdots, \ \boldsymbol{a_N} \boldsymbol{q}- b_{N} \le 0$ hold. Each of these constraints can be considered as a separate Boolean identity. Next, we use Boolean algebra theory \cite{ref29}. Let $\left\{P_1, P_2, \cdots, P_N\right\}$ be a set of $N$ truth values. Let 
\begin{equation*}
\boldsymbol{a_i} \boldsymbol{q}- b_i \le 0, \iff  P_i = \tx{true}, \ \ i=1,\cdots, N.
\end{equation*}
Then, $\tx{q}\ \cap \ \mathcal C \ne \emptyset$ equals to the following conjunction holding 
\begin{equation*}
P_1 \land  P_2 \land \cdots \land P_N = \tx{true}. 
\end{equation*}
Now we negate the conjunction based on Lemma~\ref{lemma1} as
\begin{equation}\label{eq4}
\begin{aligned}
\lnot \left(P_1 \land  P_2 \land \cdots \land P_N\right) \iff 
\left(\lnot P_1\right)\lor\left(\lnot P_2\right)\lor\cdots\lor\left(\lnot P_N\right),
\end{aligned}
\end{equation}
where $\land$, $\lor$, $\lnot$ are the AND, OR, NOT operator. From here, (\ref{eq4}) can be substituted by stating
\begin{equation}\label{eq5}
\tx{q} \cap \mathcal C =\emptyset
\iff \exists \  \boldsymbol{a_i} \boldsymbol{q}- b_i > 0,
\ \ i=1,\cdots, N. 
\end{equation}
This states that point q is outside the polygon $\mathcal C$, iff there is at least one constraint where the inequality in \eqref{eq5} is satisfied. To approach this, let the algorithm find a sufficiently large variable $\lambda_i \geq 0$ that multiples the constraint that is satisfied, i.e., $(\boldsymbol{a_i} \boldsymbol{q}- b_i)\lambda_i > 0$, and a sufficiently small $\lambda_j \geq 0$ that  multiples the constraint that is not satisfied, i.e., $(\boldsymbol{a_j} \boldsymbol{q}- b_j)\lambda_j \leq 0$. Then, summing them all together ensures the sum is positive, i.e., 
\begin{equation}\label{eq6}
(\boldsymbol{a_1} \boldsymbol{q}- b_1)\lambda_1+(\boldsymbol{a_2} \boldsymbol{q}- b_2)\lambda_2+\cdots+(\boldsymbol{a_N} \boldsymbol{q}- b_{N}) \lambda_{N}> 0.
\end{equation}
Let $\boldsymbol{\lambda}=[\lambda_1,\lambda_2,\cdots,\lambda_N]^{\top}$. Then, constraint (\ref{eq6}) 
can be rewritten as 
\begin{equation}\label{eq7}
(A \boldsymbol{q}-\boldsymbol{b})^{\top}\boldsymbol{\lambda} > 0,\  \exists 
 \boldsymbol{\lambda}\geq \boldsymbol{0},
 \end{equation} 
which proves Proposition~\ref{prop1}.
\end{proof}

\section{Collision avoidance constraints}\label{sec4}
In this section, We first introduce mathematical representations of vehicle, obstacles, and the driving environment. Then collision constraints (\ref{eq2e}), (\ref{eq2f}) are reformulated as explicit and differential constraints based on Proposition~\ref{prop1} and the following theorem:
\begin{theorem}[Hyperplane separation theorem]\label{theorem1}
    If two convex subsets of $\mathbb {R}^{2}$ are closed and at least one of them is compact, then the requirement that they are disjoint is identical to stating that there are two parallel hyperplanes in between them separated by a gap \cite{ref19, ref30}.
\end{theorem}

\subsection{Vehicle, obstacle and environment modeling}
Let the environment be modelled as a convex polygon $\mathcal W$ described by a linear matrix inequality as 
\begin{equation}\label{eq8}
\mathcal{W}:=\left\{\boldsymbol{p} \in \mathbb{R}^2 \mid E \boldsymbol{p} \leq \boldsymbol{f}\right\},
\end{equation}
where $E \in \mathbb{R}^{N_{\mathcal W} \times 2}$, $\boldsymbol{f} \in \mathbb{R}^{N_{\mathcal W}}$.

In contrast to the environment, a vehicle or obstacle is modeled as a more general union of polygons, denoted as $\mathcal B=\bigcup \mathcal B_m$, $\mathcal O=\bigcup \mathcal O_n $. We use a union, since the polygons may be disjoint. Then constraints \eqref{eq2f} are formulated as
\begin{equation}\label{eq9}
{\mathcal B_m(\boldsymbol{\xi}) \cap \mathcal O_n=\emptyset}, \ \forall  m, n.
\end{equation}
 
Typically, each $\mathcal O_n$ is a nonconvex polygon while each $\mathcal B_m$ can either be a convex or nonconvex polygon. For generality, we assume each $\mathcal B_m$ as a nonconvex polygon. When considering collision avoidance for nonconvex polygons, one prevalent way is decomposing them into a union of convex polygons  \cite{ref10, ref11}. Following this idea, polygon $\mathcal B_m$ or $\mathcal O_n$ is formulated as a union of convex polygons, defined as $\mathcal B_m=\bigcup \mathcal B_m^u$, $\mathcal  O_n=\bigcup \mathcal O_n^\nu $. Thus, (\ref{eq9}) can be substituted with
\begin{equation}\label{eq10}
{\mathcal B_m^u(\boldsymbol{\xi}) \cap \mathcal O_n^\nu=\emptyset}, \ \forall  m, n,u,\nu.
\end{equation}
Here $\mathcal B_m^u$ and $\mathcal O_n^\nu$ represent two convex sets. Let $\mathcal B_m^u$ be denoted as an ordered linear matrix inequality 
\begin{equation}\label{eq11}
\mathcal B_m^u :=\left\{\boldsymbol{q}(\boldsymbol{\xi}) \in \mathbb{R}^2 \mid \ C  \boldsymbol{q}(\boldsymbol{\xi}) \leq \boldsymbol{d}\right\},
\end{equation}
where $C=\left[\boldsymbol{c_1}, \ldots, \boldsymbol{c_{N_m^u}}\right]^{\top} \in \mathbb{R}^{N_m^u \times 2}$ and $\boldsymbol{d} \in \mathbb{R}^{N_m^u}$.  Here, $\boldsymbol{c_i^{\top}}$ and $d_i$ denote the {\it i}-th  row vector in $C$ and the {\it i}-th entry in $\boldsymbol{d}$ respectively. The convex polygon $\mathcal O_n^\nu$ is represented according to (\ref{eq3}), i.e.,
\begin{equation}\label{eq12}
\mathcal O_n^\nu:=\left\{\boldsymbol{q} \in \mathbb{R}^2 \mid A \boldsymbol{q} \leq \boldsymbol{b}\right\},
\end{equation}
where $A=\left[\boldsymbol{a_1}, \ldots, \boldsymbol{a_{N_n^\nu}}\right]^{\top} \in \mathbb{R}^{N_n^\nu \times 2}$ and $\boldsymbol{b} \in \mathbb{R}^{N_n^\nu}$. $\boldsymbol{a_j^{\top}}$ and $b_j$ are the {\it j}-th row vector in $A$ and the {\it j}-th entry in $\boldsymbol{b}$. 

Let $V=\left[\boldsymbol{v_1}, \ldots, \boldsymbol{v_{N_m^u}}\right]\in\mathbb R^{2\times N_m^u}, O=\left[\boldsymbol{o_1}, \ldots, \boldsymbol{o_{N_n^\nu}}\right]\in\mathbb R^{2\times N_n^\nu}$ return two matrices of vertices in $\mathcal B_m^u$ and $\mathcal O_n^\nu$, respectively. Symbols  $\boldsymbol{v_i}$ and $\boldsymbol{o_j} $ describe the positions of the {\it i}-th vertex and {\it j}-th vertex of $\mathcal B_m^u$ and $\mathcal O_n^\nu$, respectively. Vertex $\boldsymbol{v_i}$ can be obtained by  $\boldsymbol{c_i}^{\top} \boldsymbol{v_i} = d_i, \boldsymbol{c_{i+1}}^{\top} \boldsymbol{v_i} = d_{i+1}, i=1,\cdots, N_m^u-1$, $\boldsymbol{c_i}^{\top} \boldsymbol{v_i} = d_i, \boldsymbol{c_1}^{\top} \boldsymbol{v_i} = d_{1}, i=N_m^u$; $\boldsymbol{o_j} $ is calculated by $\boldsymbol{a_j}^{\top} \boldsymbol{o_j} = b_j, \boldsymbol{a_{j+1}}^{\top} \boldsymbol{o_j} = b_{j+1}, j=1,\cdots, N_n^\nu-1;\boldsymbol{a_j}^{\top} \boldsymbol{o_j} = b_j, \boldsymbol{a_{1}}^{\top} \boldsymbol{o_j} = b_{1}, j= N_n^\nu$.

\subsection{Collision avoidance formulations}
In the OCP (\ref{eq2}), constraint (\ref{eq2e}) enforces the vehicle to not collide with the environment boundary, i.e., each vertex of $\mathcal B_m^u$ must stay within $\mathcal W$. Since $\mathcal W$ is convex, (\ref{eq2e}) can be formulated as a linear inequality
\begin{equation}\label{eq13}
E \boldsymbol{v_i}\leq \boldsymbol{f},\ \forall i,m, u.
\end{equation}
Formulating the constraint (\ref{eq10}) between the vehicle and obstacles in a smooth form requires multiple steps, but it can be achieved in different ways. Next, we introduce three exact smooth formulations of this constraint. 

The sole criterion for achieving $\mathcal B_m^u(\boldsymbol{\xi}) \cap \mathcal O_n^\nu=\emptyset$ is identical to stating that any vehicle polygon $\mathcal B_m^u$ does not overlap with an obstacle polygon $\mathcal O_n^\nu$. When $\mathcal B_m^u$ intersects with $\mathcal O_n^\nu$, some vertices of $\mathcal B_m^u$ are bound to enter $\mathcal O_n^\nu$, or conversely, vertices belonging to $\mathcal O_n^\nu$ may enter $\mathcal B_m^u$. In order to avoid collisions, the vertices of $\mathcal B_m^u$ and $\mathcal O_n^\nu$ should be kept outside each other. To approach this, vertex constraints are first constructed to separate a vertex of $\mathcal B_m^u$ from $\mathcal O_n^\nu$, and equivalently, a vertex of $\mathcal O_n^\nu$ from $\mathcal B_m^u$ by exploiting Proposition~\ref{prop1}. Then, based on vertex constraints, $\mathcal B_m^u(\boldsymbol{\xi}) \cap \mathcal O_n^\nu=\emptyset$ can be formulated by checking all vertices of the vehicle and obstacles.

\begin{prop}\label{prop2}
Let $\mathcal B_m^u$ and $\mathcal O_n^\nu$ be denoted by (\ref{eq11}) and (\ref{eq12}), respectively. Then, 
\begin{equation}\label{eq14}
\begin{aligned}
\mathcal B_m^u(\boldsymbol{\xi}) \cap \mathcal O_n^\nu=\emptyset \iff \exists \Lambda  \succeq 0_{N_m^u, N_n^\nu}, \Omega \succeq  0_{N_n^\nu, N_m^u}: \\
\begin{array}{c}
(C O- \boldsymbol{d}\boldsymbol{1}^{\top})^{\top} \Lambda \succ 0_{N_n^\nu, N_n^\nu},  \\
(A V- \boldsymbol{b}\boldsymbol{1}^{\top})^{\top} \Omega \succ 0_{N_m^u, N_m^u},
\end{array}
\end{aligned}
\end{equation}
where $\boldsymbol{1}=(1,\cdots,1)^\top$ is a vector with all 1, and $0_{(\cdot),(\cdot)}$ be a matrix with all 0. Symbols $\succeq$ , $\succ$ represent elementwise inequalities.
\end{prop}
Based on Proposition~\ref{prop1}, (\ref{eq14}) explicitly formulates constraints (\ref{eq10}) by introducing auxiliary variables to be optimized associated with vertices of $\mathcal B_m^u$ and $\mathcal O_n^\nu$. 

Another way of reformulating constraints (\ref{eq10}) is by exploiting Theorem~\ref{theorem1}. 

\begin{lemma}\label{lemma2}
Let $S \subset \mathbb{R}^n$. Then the following two statements are equivalent \cite{ref31}:
\begin{enumerate}
\item{$S$ is closed and bounded.} 
\item{$S$ is compact, that is, every open cover of $S$ has a finite subcover.}
\end{enumerate}
\end{lemma}
In formulations (\ref{eq11}) and (\ref{eq12}), $\mathcal B_m^u$ and $\mathcal O_n^\nu$ are two closed and bounded sets. According to Lemma~\ref{lemma2}, these sets are also compact. 
What's more, based on Theorem~\ref{theorem1}, when set $\mathcal B_m^u$ and set $\mathcal O_n^\nu$ do not intersect, there exists a separating gap between them, i.e., points belonging to $\mathcal B_m^u$ are on one side of the gap while points of $\mathcal O_n^\nu$ are on the other side. However, checking all points of $\mathcal B_m^u$ and $\mathcal O_n^\nu$ is computationally expensive. One efficient way is only checking vertices of $\mathcal B_m^u$ and $\mathcal O_n^\nu$. 
\begin{lemma}\label{lemma3}
Any interior point of a convex polygon can be expressed as a convex combination of its vertices, where all coefficients are non-negative and sum to 1 \cite{ref32}.
\end{lemma}
\begin{prop}\label{prop3}
Let $\mathcal V$ represent the set of vertices in $\mathcal B_m^u$ and let $\mathcal H^{-}(\boldsymbol{\lambda}, \mu)=\left\{\boldsymbol{q} \in \mathbb{R}^2 \mid \boldsymbol{\lambda}^\top \boldsymbol{q} < \mu\right\}$ be the half-space with outward normal vector $\boldsymbol{\lambda}$. Then,
\begin{equation}\label{eq15}
\begin{aligned}
\mathcal V \subset \mathcal H^{-}(\boldsymbol{\lambda}, \mu) \iff \mathcal B_m^u \subset \mathcal H^{-}(\boldsymbol{\lambda}, \mu).
\end{aligned}
\end{equation}
\end{prop}
\begin{proof}
For the set of vertices, we have $\mathcal V \subset \mathcal B_m^u$. Then, if $\mathcal B_m^u \subset \mathcal H^{-}(\boldsymbol{\lambda}, \mu)$ it follows that $\mathcal V \subset \mathcal H^{-}(\boldsymbol{\lambda}, \mu)$. 
This proves the right-to-left implication. The left-to-right implication, given $\mathcal V \subset \mathcal H^{-}(\boldsymbol{\lambda}, \mu)$ will be proven by contradiction.
Assume $\exists\boldsymbol{q} \in \mathcal B_m^u $, $ \boldsymbol{q}\notin \mathcal H^{-}(\boldsymbol{\lambda}, \mu)$, such that $\boldsymbol{\lambda}^\top \boldsymbol{q} \geq \mu$. Based on Lemma~\ref{lemma3}, $\boldsymbol{q}=\sum_{i=1}^{N_m^u}\alpha_i \boldsymbol{v_i}$, where $\alpha_i \geq 0$. Then, $\sum_{i=1}^{N_m^u}\alpha_i \boldsymbol{\lambda}^\top \boldsymbol{v_i} \geq \mu$,  indicates that $\exists i, \boldsymbol{\lambda}^\top \boldsymbol{v_i} \geq \mu$, i.e., $\boldsymbol{v_i} \notin \mathcal H^{-}(\boldsymbol{\lambda}, \mu)$, which contradicts with $\mathcal V \subset \mathcal H^{-}(\boldsymbol{\lambda}, \mu)$. So, it follows that $\forall \boldsymbol{q} \in \mathcal B_m^u$, $\boldsymbol{q} \in \mathcal H^{-}(\boldsymbol{\lambda}, \mu)$, i.e., $ \mathcal B_m^u \subset \mathcal H^{-}(\boldsymbol{\lambda}, \mu)$, which completes the proof. 
\end{proof}
Using Proposition~\ref{prop3}, the collision constraints (\ref{eq10}) can be reformulated by ensuring that the vertices of $\mathcal B_m^u$ are on one side of a separating gap and vertices of $\mathcal O_n^\nu$ are on the other side.

\begin{prop}\label{prop4}
Let $V$ and $O$ represent matrices of all the vertices of $\mathcal B_m^u$ and $\mathcal O_n^\nu$, respectively. Then, 
\begin{equation}\label{eq16}
\begin{aligned}
\mathcal B_m^u(\boldsymbol{\xi}) \cap \mathcal O_n^\nu=\emptyset \iff  \exists \boldsymbol{\lambda} \in \mathbb{R}^2, \mu_1, \mu_2 \in \mathbb{R}: \ \ \ \ \ \ \ \ \   \\ 
  \begin{array}{c}
\boldsymbol{\lambda}^\top V > \mu_1 \boldsymbol{1}^\top, \boldsymbol{\lambda}^\top O < \mu_2 \boldsymbol{1}^\top, \mu_1 > \mu_2, \left\|\boldsymbol{\lambda}\right\| > 0.
\end{array}
\end{aligned}
\end{equation}
\end{prop}
\begin{proof}
Based on Theorem~\ref{theorem1}, $\mathcal B_m^u(\boldsymbol{\xi}) \cap \mathcal O_n^\nu=\emptyset$ is identical to stating that they are separated by a gap surrounded by two parallel hyperplanes. It is noticed that the gap is an unbounded convex polygon. Let ${\mathcal H^{+}(\boldsymbol{\lambda}, \mu_1)=\left\{\boldsymbol{q} \in \mathbb{R}^2 \mid \boldsymbol{\lambda}^\top \boldsymbol{q} > \mu_1\right\}}$ be one half-space not including the gap. Let  $\mathcal H^{-}(\boldsymbol{\lambda}, \mu_2)=\left\{\boldsymbol{q} \in \mathbb{R}^2 \mid \boldsymbol{\lambda}^\top \boldsymbol{q} < \mu_2\right\}$ be the other half-space not including the gap. We assume $\mu_1 > \mu_2$. According to Proposition~\ref{prop3}, $\mathcal B_m^u(\boldsymbol{\xi}) \cap \mathcal O_n^\nu=\emptyset$ it follows that vertices of $\mathcal B_m^u$ stay in $\mathcal H^{+}(\boldsymbol{\lambda}, \mu_1)$ and vertices of $\mathcal O_n^\nu$ in $\mathcal H^{-}(\boldsymbol{\lambda}, \mu_2)$, i.e., $\boldsymbol{\lambda}^\top V > \mu_1 \boldsymbol{1}^\top, \boldsymbol{\lambda}^\top O < \mu_2 \boldsymbol{1}^\top$. Additionally, $\left\|\boldsymbol{\lambda}\right\| > 0$ is imposed to ensure $\mu_1 > \mu_2$.
\end{proof}

The third way of reformulating constraints (\ref{eq10}) is a minor variation of the second approach. Instead of enforcing vertices of $\mathcal B_m^u$ and $\mathcal O_n^\nu$ to reside in two sides of a separating polygon, it is sufficient to enforce the separation by a single hyperplane.

\begin{prop}\label{prop5}
Let $V$ and $O$ be defined as stated earlier. Then, 
\begin{equation}\label{eq17}
\begin{aligned}
\mathcal B_m^u(\boldsymbol{\xi}) \cap \mathcal O_n^\nu=\emptyset \iff  \exists \boldsymbol{\lambda} \in \mathbb{R}^2 , \mu \in \mathbb{R}: \ \ \ \ \ \ \ \ \  \ \ \ \   \\ 
  \begin{array}{c}
\boldsymbol{\lambda}^\top V > \mu \boldsymbol{1}^\top, \boldsymbol{\lambda}^\top O < \mu \boldsymbol{1}^\top, \left\|\boldsymbol{\lambda}\right\| > 0.
\end{array}
\end{aligned}
\end{equation}
\end{prop}
\begin{proof}
We first prove the left-to-right implication. According to Proposition~\ref{prop4}, if $\mathcal B_m^u(\boldsymbol{\xi}) \cap \mathcal O_n^\nu=\emptyset$, $\exists \mu_1 > \mu_2$, $\boldsymbol{\lambda}^\top V > \mu_1 \boldsymbol{1}^\top, \boldsymbol{\lambda}^\top O < \mu_2 \boldsymbol{1}^\top$. Let $\mu_1 > \mu > \mu_2 $, it can be seen that both $\boldsymbol{\lambda}^\top V > \mu \boldsymbol{1}^\top$ and $\boldsymbol{\lambda}^\top O < \mu \boldsymbol{1}^\top$ are correct, and additionally $\left\|\boldsymbol{\lambda}\right\| > 0$ guarantees that at least one element of $\boldsymbol{\lambda}$ is nonzero. For the right-to-left implication, consider $\boldsymbol{\lambda}^\top V > \mu \boldsymbol{1}^\top, \boldsymbol{\lambda}^\top O < \mu \boldsymbol{1}^\top, \left\|\boldsymbol{\lambda}\right\| > 0$ holds. Then, based on Proposition~\ref{prop3} constraints $\boldsymbol{\lambda}^\top \boldsymbol{q} > \mu, \boldsymbol{\lambda}^\top \boldsymbol{p}< \mu$ hold $\forall \boldsymbol{q} \in \mathcal B_m^u$, $\forall \boldsymbol{p} \in \mathcal O_n^\nu$, i.e., $B_m^u(\boldsymbol{\xi}) \cap \mathcal O_n^\nu=\emptyset$.
\end{proof}
To summarize, the collision constraints (\ref{eq2e}) are reformulated through (\ref{eq14}), (\ref{eq16}), and (\ref{eq17}) by checking all $m$, $n$ ,$u$, and $v$. These constraints guarantee that the vehicle avoids collisions effectively with obstacles.

\section{Analytical evaluation in problem size}\label{sec5}
In this subsection, a few prominent methods in literature to model collision constraints (\ref{eq10}) are used as a benchmark to compare with the proposed methods in this paper. 

In \cite{ref19, ref21}, the collision constraints for a full vehicle body are derived by using indicator functions or Farkas' lemma. Let $\mathcal B_m^u$ and $\mathcal O_n^\nu$ be represented by (\ref{eq11}) and (\ref{eq12}), respectively. The constraints (\ref{eq10}) are then formulated as
\begin{equation}\label{eq18}
\begin{aligned}
\mathcal B_m^u(\boldsymbol{\xi}) \cap \mathcal O_n^\nu=\emptyset \ \ \ \ \ \ \ \ \ \ \ \ \ \ \ \ \ \ \ \ \ \ \ \ \ \ \ \ \ \ \ \ \ \ \ \ \ \ \ \\ 
\iff \exists \boldsymbol{\lambda},  \boldsymbol{\mu}  \geq \boldsymbol{0}: -\boldsymbol{b}^{\top} \boldsymbol{\lambda}-\boldsymbol{d}^{\top} \boldsymbol{\mu} \geq d_\tx{safe},\\
A^{\top} \boldsymbol{\lambda} + C^{\top} \boldsymbol{\mu}  =\mathbf{0},
\end{aligned}
\end{equation}
where $d_\tx{safe}\in \mathbb{R^+}$ is a safety distance. The method in \cite{ref21} is very similar and includes one additional constraint $\Vert A^{\top} \boldsymbol{\lambda} \Vert_2 \leq 1$. The collision constraints (\ref{eq10}) are formulated as
\begin{equation}\label{eq19}
\begin{aligned}
\mathcal B_m^u(\boldsymbol{\xi}) \cap \mathcal O_n^\nu=\emptyset \ \ \ \ \ \ \ \ \ \ \ \ \ \ \ \ \ \ \ \ \ \ \ \ \ \ \ \ \ \ \ \ \ \ \ \ \ \ \ \\ 
\iff \exists \boldsymbol{\lambda},  \boldsymbol{\mu}  \geq \boldsymbol{0}: -\boldsymbol{b}^{\top} \boldsymbol{\lambda}-\boldsymbol{d}^{\top} \boldsymbol{\mu} \geq d_\tx{safe},\\
\Vert A^{\top} \boldsymbol{\lambda} \Vert_2 \leq 1,
A^{\top} \boldsymbol{\lambda} + C^{\top} \boldsymbol{\mu} =\mathbf{0}.
\end{aligned}
\end{equation}

In \cite{ref10}, the authors first formulate collision avoidance based on the notion of signed distance, and then they reformulate the signed distance via dualization techniques. 
Let
\begin{equation}\label{eq20}
\begin{aligned}
\mathcal B_m^u(\boldsymbol{\xi}) =R(t)B_0+\boldsymbol{T}(t), B_0:=\left\{ \boldsymbol{q} \in \mathbb{R}^2 \mid \ C_0  \boldsymbol{q} \leq \boldsymbol{d}_0 \right\},
\end{aligned}
\end{equation}
where $B_0$ is the initial space occupied by polygon $\mathcal B_m^u$, $C_0 \in \mathbb{R}^{N_m^u \times 2}$ and $\boldsymbol{d}_0 \in \mathbb{R}^{N_m^u}$. $R(t)\in \mathbb{R}^{2\times 2} $ is
a rotation matrix, and $\boldsymbol{T}(t)\in \mathbb{R}^{2}$ is the translation vector. Polygon $\mathcal O_n^\nu$ is represented by (\ref{eq12}). Constraints (\ref{eq10}) are reformulated as
\begin{equation}\label{eq21}
\begin{aligned}
\mathcal B_m^u(\boldsymbol{\xi}) \cap \mathcal O_n^\nu=\emptyset \ \ \ \ \ \ \ \ \ \ \ \ \ \ \ \ \ \ \ \ \ \ \ \ \ \ \ \ \ \ \ \ \ \ \ \ \ \ \ \\ 
\iff \exists \boldsymbol{\lambda},  \boldsymbol{\mu}  \geq \boldsymbol{0}:-\boldsymbol{d}_0^{\top} \boldsymbol{\mu}+(A\boldsymbol{T}(t)-\boldsymbol{b})^{\top} \boldsymbol{\lambda} \geq d_\tx{safe},\\
\Vert A^{\top} \boldsymbol{\lambda} \Vert_2 \leq 1, C_0^{\top} \boldsymbol{\mu} + R(t)^{\top}A^{\top} \boldsymbol{\lambda} =\mathbf{0}.
\end{aligned}
\end{equation}

Table~\ref{tab5} summarizes the number of auxiliary variables needed per obstacle to formulate $\mathcal B_m^u(\boldsymbol{\xi}) \cap \mathcal O_n^\nu=\emptyset$ by different formulations. It can be seen that the proposed formulations \eqref{eq16} need 3 variables to construct the separating hyperplane between $\mathcal B_m$ and $ \mathcal O_n$, and the proposed formulations \eqref{eq17} need 4 variables to construct the separating gap, respectively. The numbers are not depending on the complexity of set $\mathcal B_m $ or $\mathcal O_n$. As a contrast, the proposed formulations \eqref{eq14}, the existing formulations \eqref{eq18}, \eqref{eq19}, and \eqref{eq21} identify additional variables proportional to the complexity of $\mathcal B_m$ and $ \mathcal O_n$ being represented, always needing more variables than the proposed formulations \eqref{eq16} and \eqref{eq17}. Table~\ref{tab5} also states that when the number of edges of $\mathcal B_m$ and $ \mathcal O_n$ increases, more variables are needed to describe formulations \eqref{eq14}, \eqref{eq18}, \eqref{eq19}, and \eqref{eq21}, while the number of needed variables remains unchanged using the proposed formulations \eqref{eq16} and \eqref{eq17}. 

\begin{table}
\centering
\caption{The number of auxiliary variables needed for per obstacle by different methods.}
    \begin{threeparttable}
\begin{tabular}{ccccccc} \toprule
 \eqref{eq14} &  \eqref{eq16} & \eqref{eq17} & \eqref{eq18} & \eqref{eq19} & \eqref{eq21}\\
\midrule
$N_m^u+N_n^\nu \tnote{1}$ & 3 & 4 & $N_m^u+N_n^\nu$ & $N_m^u+N_n^\nu$ & $N_m^u+N_n^\nu$ \\
 \bottomrule
\end{tabular}
\begin{tablenotes}   
        \footnotesize 
        \item[1] The least number of edges to construct a bounded polygon is 3, i.e., $N_{m}^u\geq 3, N_{n}^\nu\geq 3$, indicating that the number using the proposed formulations \eqref{eq16} and \eqref{eq17} is always less than that using the proposed formulations \eqref{eq14}, the existing formulations \eqref{eq18}, \eqref{eq19}, and \eqref{eq21}.
      \end{tablenotes}
    \end{threeparttable}
\label{tab5}
\end{table}

\section{Nonlinear programming formulation}\label{sec6}
The OCP (\ref{eq2}) is generally reformulated as an NLP problem that can be solved using off-the-shelf solvers. The maximum time \( t_{f}\) will be discretized into \( {k}_{f}+1\) parts, i.e., \( t_{f} = {k}_{f} \Delta t\), where $\Delta t$ is the sample interval. The travel time $t$ is replaced with \( t = {k} \Delta t\), where ${k}=0$, $\ldots$, \( {k}_{f}\). Let \(\tilde{f}\), and \(\tilde{\ell}\) denote the discretized  model dynamics \( f\) and stage cost \(\ell\). Thus, the OCP (\ref{eq2}) involving collision avoidance can be formulated as an NLP problem
\begin{subequations}\label{eq22}
\begin{align}
\underset {\boldsymbol{\xi}, \boldsymbol{u},(\cdot)} {\text{min}} \ \ &  \sum_{{k}=0}^{{k}_f} \tilde{\ell}\left(\boldsymbol{\xi}({k}), \boldsymbol{u}({k})\right) \label{eq22a}\\
\text {s.t.}  \ \ & \boldsymbol{\xi}(0)=\boldsymbol{\xi_{\tx{init}}},\ \boldsymbol{\xi}(k_\tx{f})=\boldsymbol{\xi_{\tx{final}}} \label{eq22b}\\
 & \boldsymbol{\xi}(k+1) = \tilde{f}\left(\boldsymbol{\xi}(k),\boldsymbol{u}(k)\right), \label{eq22c}\\
 & \boldsymbol{\xi}(k) \in \mathcal{X},\ \boldsymbol{u}(k) \in \mathcal{U}, \label{eq22d}\\
 & (\ref{eq13}), \forall m, u, \label{eq22e} \\ 
 & \tx{Collision \ constraints}, \forall m, n, u, \nu, 
 \label{eq22f}
\end{align}
\end{subequations}
where collision constraints refer to one of (\ref{eq14}), (\ref{eq16}), (\ref{eq17}), (\ref{eq18}), (\ref{eq19}), and (\ref{eq21}). Symbol \( (\cdot )\) is a shorthand notation for decision variables that need to be added to the vector of existing optimization variables, e.g., $\mu$ in (\ref{eq16}). 

A special case, e.g., formulating a parking problem, is when time needs to be minimized, i.e., when $t_f$ is not known in advance. This can easily be handled, e.g., by introducing a new independent variable $\tau \in [0, 1]$, and expressing the time as $t=t_f \tau$, where $t_f\geq 0$ is a scalar optimization variable that needs to be added to the problem. The state and control trajectories will become a function of $\tau$, with state dynamics 
\[ \frac{\tx{d}\boldsymbol{\xi}(\tau)}{\tx{d}\tau} = t_f f(\boldsymbol{\xi}(\tau), \boldsymbol{u}(\tau)) \]
and a generalized cost function
\[ \underset {\boldsymbol{\xi}, \boldsymbol{u}, t_f,(\cdot)}{\text{min}} \ t_f\left(r + \int_0^{1}  
 \ell\left(\boldsymbol{\xi}(\tau), \boldsymbol{u}(\tau)\right)\tx{d}\tau \right) \]
where $r$ is a penalty factor that trades the parking time with the rest of the cost. A discretization procedure can then similarly be applied as proposed in the previous paragraph. 

\section{Numerical simulations}\label{sec7}
The trajectory optimization problems for autonomous vehicles parking in a limited maneuverable space are considered in this part. Various tests are implemented in vertical parking, parallel parking, and oblique parking scenarios. We validate the effectiveness of the proposed methods in (\ref{eq14}), (\ref{eq16}), (\ref{eq17}), by comparing with formulations (\ref{eq18}), (\ref{eq19}), (\ref{eq21}) in the NLP framework (\ref{eq22}). In the following, the detailed form of the resulting NLP is first presented, and three different initial guesses are proposed. Then, we discuss the solution performance and problem size of each formulation in all scenarios.

\subsection{Dynamic model and cost function}
Since the parking scenarios involve low-speed maneuvers, a kinematic vehicle model \cite{ref5, ref10} is exploited. The system dynamics is modeled as
\begin{equation}\label{eq23}
\begin{array}{l}
    \dot{x}=v \cos (\theta),\\
    \dot{y}=v \sin (\theta),\\
    \dot{\theta}=v\tan(\delta)/L,\\
    \dot{v}=a,\\
    \dot{\delta}=\omega,
\end{array}
\end{equation}
where $L=\SI{2.796}{m}$ is the wheelbase of the vehicle. The position of the middle of the rear axle of the vehicle $\mathcal B$ is identical to the vector $\left(x,y\right)$, and $\theta$ represents the yaw angle with respect to the horizontal axle. The velocity and acceleration of the rear axle are $v$ and $a$. The steering angle and the gradient of the steering angle are denoted by $\delta$ and $\omega$, respectively. The state and control vectors can be summarized as
\begin{equation}\label{eq24}
\begin{array}{c}
\boldsymbol{\xi} = \left[x,y,\theta,v,\delta\right]^{\top},
\boldsymbol{u} = \left[a,\omega\right]^{\top}. \\
\end{array}
\end{equation}
The limits of states $x,y$ in each scenario are shown in Table~\ref{tab1}. In all parking simulations, the feasible values of $\theta,v,\delta$ are limited to $|\theta| \leq \SI{180}{\degree}$, $ |v| \leq \frac{5}{3.6} \si{m/s^2}$, $|\delta| \leq 40 \si{\degree}$. Additionally, the feasible inputs are given by $|a|\leq \SI{1}{m/s^2}$ and $|\omega|\leq \SI{5}{\degree/s^2}$.

We expect the autonomous vehicle to complete the parking maneuvers subjected to minimum traveling time. Also, the control input $a$ and $\omega$ should be small to yield smooth trajectories. Thus, a typical time-energy cost function $J$ is formulated as 
\begin{equation}\label{eq25}
\begin{aligned}
J(\boldsymbol{u},t_f) 
= t_f \left(r + \int_{0}^{1} \left\Vert \boldsymbol{u}\left(\tau\right)\right\Vert_{P}^{2} \tx{d}\tau\right).
\end{aligned}
\end{equation}
The notation \(\left\Vert \cdot\right\Vert_{P}^{2}\) represents squared Euclidean norm weighted by matrix \(P\). For the studied cases we choose $r=1, P={\rm diag}(1, 2)$.

\subsection{Case descriptions}
In the example of Fig.~\ref{fig1}, three different parking scenarios are considered. The driving environment $\mathcal W$ is modeled by a convex polygon, depicted in green. In each scenario, two obstacles $\mathcal O_1$ and $\mathcal O_2$ are modeled by convex polygons, depicted in red solid. The vehicle $\mathcal B$ is modeled as a rectangle of size $\SI{4.628}{m} \times \SI{2.097}{m}$, depicted in purple dotted line. The environment is modeled as $\mathcal W :=\left\{\boldsymbol{q} \in \mathbb{R}^2 \mid A_0 \boldsymbol{q} \leq \boldsymbol{b_0}\right\}$, and each obstacle as $\mathcal O_j :=\left\{\boldsymbol{q} \in \mathbb{R}^2 \mid A_j \boldsymbol{q} \leq \boldsymbol{b_j}\right\}, j=1,2$, 
with coefficients for the various parking scenarios listed in Table~\ref{tab1}.
\begin{table}
\centering
\caption{The descriptions of the driving environment and obstacles.}
{\begin{tabular}{cccc} \toprule
Parameters & Vertical &  Parallel   & Oblique  \\
\midrule
x(m) & (-2, 15) & (-2, 22) & (-4, 20) \\
y(m) & (-8, 8) & (-6, 8) & (-8, 4) \\
$[\boldsymbol{A_0},\boldsymbol{b_0}]$&
$\begin{bmatrix}0 , 1 , 8 \\ 0 , -1 , 8\\ -1 , 0 , 2 \\ 1 , 0 , 15\end{bmatrix}$ & $\begin{bmatrix}0 , 1 , 8 \\ 0 , -1 , 6\\ -1 , 0 , 2 \\ 1 , 0 , 22\end{bmatrix}$   & $\begin{bmatrix}0 , 1 , 4 \\ 0 , -1 , 8\\ -1 , 0 , 4 \\ 1 , 0 , 20\end{bmatrix}$     \\
$[\boldsymbol{A_1},\boldsymbol{b_1}]$ & 
$\begin{bmatrix} 0, 1, -2\\  1, 0, 5\\ 0,-1, 8\\ -1, 0, 0\end{bmatrix}$ & 
$\begin{bmatrix} 0, 1, -3\\ 1, 0, 5\\ 0, -1, 6\\ -1,0, 0\end{bmatrix}$ & 
$\begin{bmatrix} 0, 1, -2\\ -1, 0, 7\\ 0,-1, 8\\ 1, 0, 2\end{bmatrix}$  \\
$[\boldsymbol{A_2},\boldsymbol{b_2}]$ & 
$\begin{bmatrix} 0, 1, -2\\-1, 0, -7.5\\ 0, -1, 8\\ 1, 0, 0\end{bmatrix}$ & 
$\begin{bmatrix} 0, 1, -3\\ 1, -1,-12 \\ 0, -1, 6\\ -1, 0, 20\end{bmatrix}$ & 
$\begin{bmatrix} 0, 1, -2\\-1, 1, -11\\ 0,-1, 8\\ 1,0, 18\end{bmatrix}$   \\
 \bottomrule
\end{tabular}}
\label{tab1}
\end{table}

Simulations were conducted in MATLAB R2020a and executed on a laptop with AMD R7-5800H CPU at 3.20 GHz and 16GB RAM. The explicit 4-th order Runge-Kutta method is used as a numerical integration method. Since we use a multiple shooting approach, the problem is formulated as a collection of $k_\tx{f}$ phases. Let $k_\tx{f}=20$. The problems are then implemented using the CasADi \cite{ref33}. The NLP (\ref{eq22}) is solved using the interior point solver IPOPT \cite{ref34}.  
\begin{remark}
In all simulation cases, we model each obstacle with four hyperplanes. Although the number of hyperplanes to model the critical sides of obstacles can be reduced to improve the efficiency, e.g., eliminating hyperplanes by culling procedures \cite{ref39}, we abstain from using it for didactic reasons since the goal is comparing the solution quality of different formulations of constraints (\ref{eq10}) on the same premise. 
\end{remark}

\begin{figure*}[!t]
\centering
\subfloat[Vertical parking]{\includegraphics[width=0.33\textwidth]{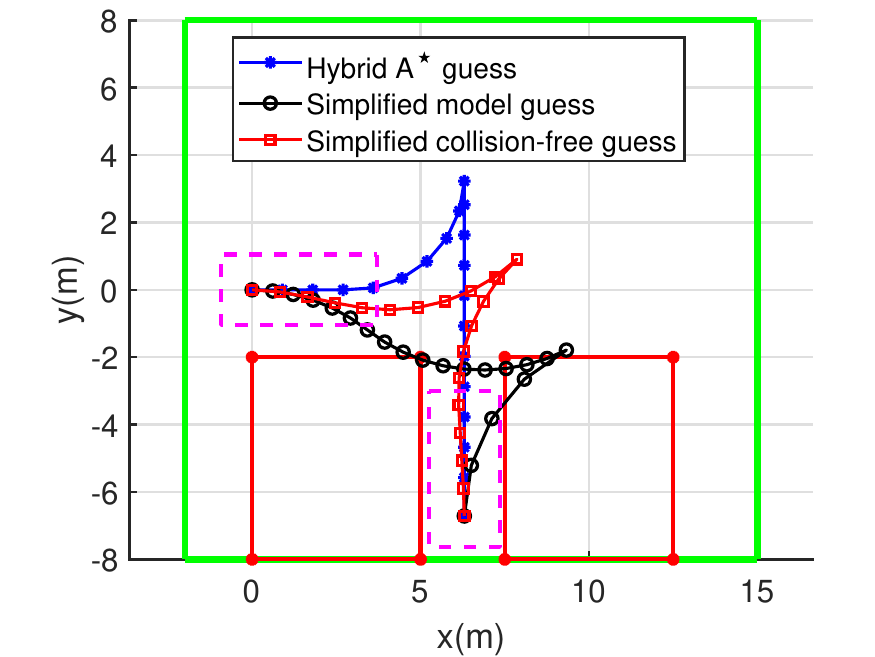}%
\label{fig1a}}
\hfil
\subfloat[Parallel parking]{\includegraphics[width=0.33\textwidth]{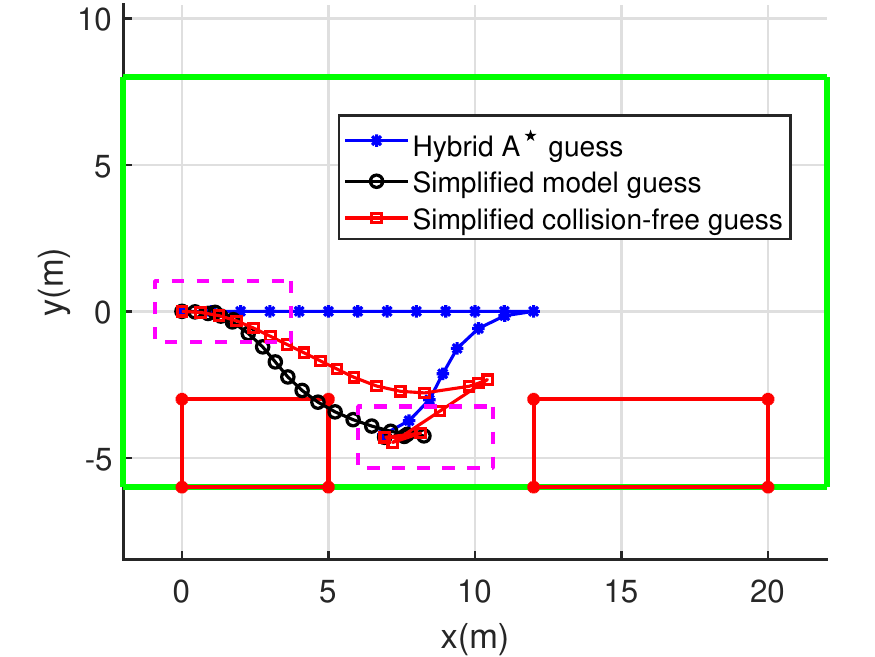}%
\label{fig1b}}
\hfil
\subfloat[Oblique parking]{\includegraphics[width=0.33\textwidth]{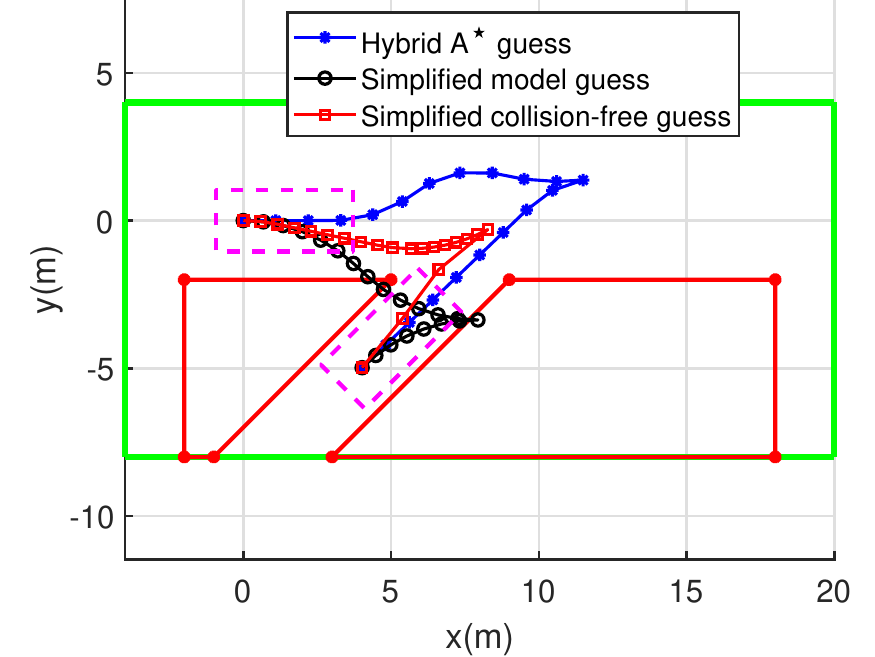}%
\label{fig1c}}
\hfil
\caption{Variable initial guesses in each parking scenario. The execution of parallel parking presents greater complexity for autonomous vehicles owing to the potential occurrence of multiple reverse behaviors.  The starting state $\boldsymbol{\xi_{\tx{init}}}$ of these three scenarios  is  $\begin{bmatrix}\SI{0}{m}, \SI{0}{m}, \SI{0}{\degree}, \SI{0}{m/s}, \SI{0}{\degree}\end{bmatrix}$. The ending state $\boldsymbol{\xi_{\tx{final}}}$ is $\begin{bmatrix}\SI{6.3}{m}, \SI{-6.7}{m},\SI{90}{\degree},\SI{0}{m/s},\SI{0}{\degree}\end{bmatrix}$,  $\begin{bmatrix}\SI{6.9}{m}, \SI{-4.3}{m}, \SI{0}{\degree},\SI{0}{m/s},\SI{0}{\degree} \end{bmatrix}$, $\begin{bmatrix}\SI{4}{m}, \SI{-5}{m},\SI{45}{\degree},\SI{0}{m/s},\SI{0}{\degree} \end{bmatrix}$, respectively. } 
\label{fig1}
\end{figure*}

\subsection{Initial Guesses}
Nonconvex problems, such as (\ref{eq22}), may have multiple local solutions. Indeed, problem \eqref{eq22} may have infinitely many local solutions, as the vehicle may take infinitely many paths to reach the goal.  
An initial guess in the neighborhood of a local solution is more likely to cause the NLP solver to return that local solution. Moreover, a bad initial guess may even prevent the solver to find a feasible solution in a reasonable time 
\cite{ref19, ref38, ref41 }. For this reason,  
we construct three initial guesses to verify the computational effort when using the different methods of modelling collision constraints. 

\subsubsection{Simplified model guess} \label{sec:simplified-guess}
Here, we use a simplified vehicle model in the parking problem, and use its solution as an initial guess to problem \eqref{eq22}. We model the vehicle dynamics with the simplified state vector $[x,y,\theta]^{\top}$ and control vector $[v,\delta]^{\top}$ in (\ref{eq24}). Except for the penalty on travel time, we expect the autonomous vehicle to avoid the long displacement between two adjacent samples, so the cost function is defined as $t_f \left(\int_{0}^{1} pv(\tau)^2 \tx{d}\tau + r\right)$, where $p$ is a weighting factor. Additionally, we implement this planning program in an environment $\mathcal W$ without obstacles, leading to the NLP
{\allowdisplaybreaks
\begin{subequations}\label{eq26}
\begin{align}
\underset {\boldsymbol{\xi}, \boldsymbol{u},t_f} {\text{min}} \ & t_f \left(\int_{0}^{1} pv(\tau)^2 \tx{d}\tau + r\right)
\label{eq26a}\\
\text { s.t.}   \ & \tx{The \ simplified \ model \ dynamics}, \label{eq26b}\\
 & \tx{ Initial \ and \ terminal \ constraints}, \label{eq26c}\\
 &  \tx{State \ and \ control \ limits}, \label{eq26d}\\
 & (\ref{eq13}). \label{eq26e}
\end{align}
\end{subequations}}

Notice that the NLP \eqref{eq26} itself needs an initial guess. For this, we simply linearly interpolate the states from their initial to target values, while the initial guess for the control inputs is set to zero. 

\subsubsection{Simplified collision-free guess}
To obtain initial guesses satisfying collision-free conditions, based on the NLP (\ref{eq26}), we add collision constraints constructed by formulation (\ref{eq18}).
The resulting NLP is 
{\allowdisplaybreaks
\begin{subequations}\label{eq27}
\begin{align}
\underset {\boldsymbol{\xi}, \boldsymbol{u},t_f,(\cdot)} {\text{min}} \ & t_f \left(\int_{0}^{1} pv(\tau)^2 \tx{d}\tau + r\right) 
\label{eq27a}\\
\text { s.t.}  \ & \tx{The \ simplified \ model \ dynamics}, \label{eq27b}\\
 & \tx{Initial \ and \ terminal \ constraints}, \label{eq27c}\\
 &  \tx{State \ and \ control \ limits}, \label{eq27d}\\
 & (\ref{eq13}), (\ref{eq18}) \label{eq27e}
\end{align}
\end{subequations}}
which is initialized exactly as that in Section~\ref{sec:simplified-guess}. 

\subsubsection{Hybrid $\tx{A}^{\star}$ guess}
The hybrid $\tx{A}^{\star}$ approach, which is widely applied in parking scenarios, is a version of the $\tx{A}^{\star}$ algorithm combined with the Reeds-Shepp curve generation. This approach uses a simplified kinematic model with state vector $[x,y,\theta]^{\top}$ while neglecting control vector $[v,\delta]^{\top}$. The control variables are replaced with a discrete input that decides if the vehicle moves forward or backward. The interested readers can get more information in \cite{ref35}. By exploiting the hybrid $\tx{A}^{\star}$ approach, the shortest collision-free path and the orientations of vehicle along the path from an initial state to a final state can be obtained.

Fig.~\ref{fig1} illustrates the initial guesses based on different methods. Generally, initial guesses that (almost) satisfy many of the constraints reduce the work involved in finding a feasible solution. It is notable that nonholonomic constraints \cite{ref36, ref37} of ground vehicles are considered in all three initial guesses. However, the initial guesses by solving (\ref{eq27}) and exploiting hybrid $\tx{A}^{\star}$ method are collision-free while the initial guess derived from (\ref{eq26}) is not. Based on these initial guesses, the solution results and computational demand of solving the NLP
(\ref{eq22}) are shown in Table~\ref{tab2}.

\subsection{Comparisons of obtained solutions }

The first observation from Table~\ref{tab2} is that based on the different initial guesses and various collision formulations, there exist a few local solutions in each scenario, especially in parallel parking. For the oblique scenario is simple and it can be seen that the objective value $J$ is exactly the same in this scenario. The parallel scenario is complex and based on a simplified model guess, the $J$, TS, and $t_f$ are widely large (depicted in red values), indicating that the solver can not converge to a good local optimum and the corresponding trajectories are less smooth and time-consuming, while based on a simplified collision-free guess or hybrid $\tx{A}^{\star}$ guess, the solver finds better local solutions. The best solutions in vertical and parallel parking refer to the objective values of 1.012 and 1.221, respectively. These solutions can be obtained by exploiting the proposed methods in (\ref{eq14}), (\ref{eq16}), and (\ref{eq17}). The corresponding trajectories are shown in Fig.~\ref{fig2a}--\ref{fig2c}. It can be seen that the parking space is narrow, and the autonomous vehicle keeps close to obstacles but the vertices of the vehicle do not enter into obstacle edges and vertices of obstacles do not cross any edge of the vehicle. These planned trajectories are collision-free and smooth. 

\subsection{Comparisons of problem size}

From the CN and VN columns in Table~\ref{tab2}, it can be seen that the NLP (\ref{eq22}) using the proposed (\ref{eq16}) or (\ref{eq17}) shows a large reduction of the number of optimization variables and inequality constraints, while the proposed method in (\ref{eq14}) needs the same number of auxiliary variables as in the existing methods (\ref{eq18}), (\ref{eq19}), and (\ref{eq21}). The reason why the values for CN and VN are lower when using (\ref{eq16}) and (\ref{eq17}) is because the separating hyperplanes method needs 3 variables ($\boldsymbol{\lambda} \in \mathbb{R}^2 , \mu \in \mathbb{R}$ in \eqref{eq17}), and 4 variables ($\boldsymbol{\lambda} \in \mathbb{R}^2 , \mu_1, \mu_2 \in \mathbb{R}$ in \eqref{eq16}) per obstacle, and variables depend on the number of obstacles. On the contrary, the existing methods (\ref{eq18}), (\ref{eq19}), (\ref{eq21}) and the proposed method (\ref{eq14}) identify variables associated with edges of a vehicle and an obstacle, depending on the number of the total edges of the vehicle and all obstacles. It states that when the number of edges of each obstacle increase, more variables are needed to formulate (\ref{eq18}), (\ref{eq19}), (\ref{eq21}), and (\ref{eq14}), while the number of variables by the proposed (\ref{eq16}) or (\ref{eq17}) remains unchanged. Typically, obstacles are various nonconvex polygons composed of many convex polygons, so this advantage of the proposed methods (\ref{eq16}) and (\ref{eq17}) are highly beneficial for efficient and exact collision avoidance in real-world scenarios.

\begin{figure*}[!t]
\centering
\subfloat[J=1.012]{\includegraphics[width=0.33\textwidth]{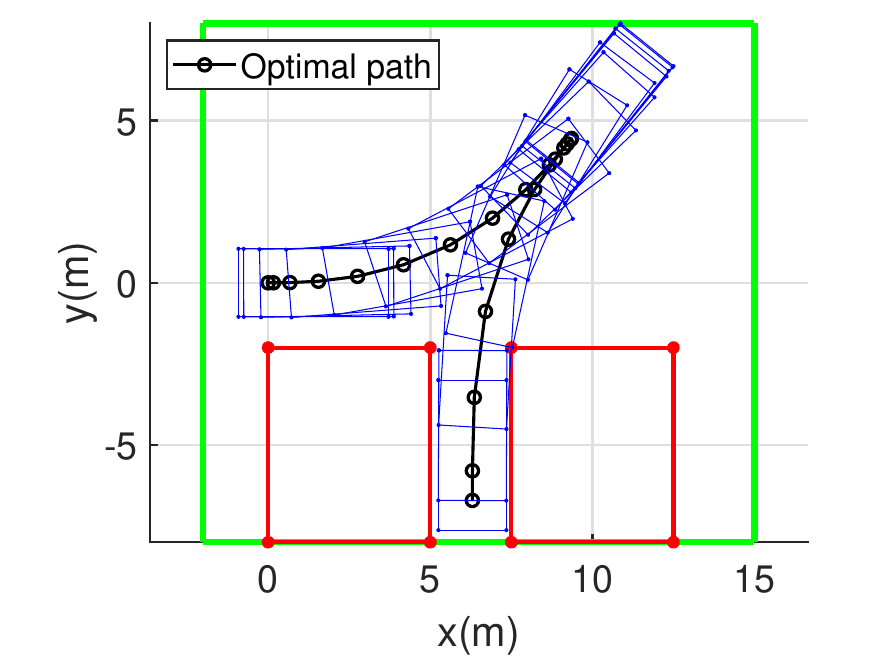}%
\label{fig2a}}
\hfil
\subfloat[J=1.221]{\includegraphics[width=0.33\textwidth]{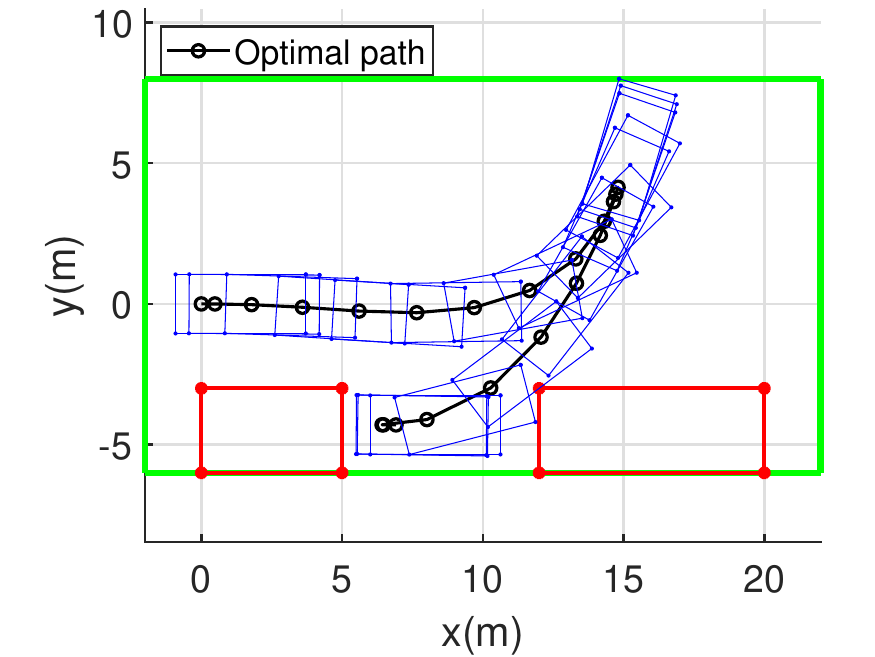}%
\label{fig2b}}
\hfil
\subfloat[J=0.856]{\includegraphics[width=0.33\textwidth]{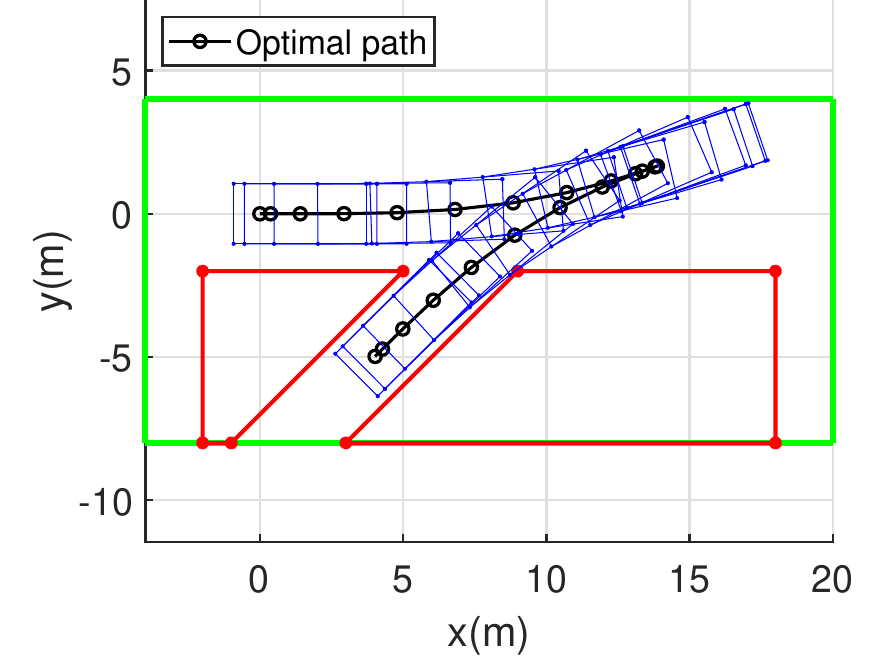}%
\label{fig2c}}
\hfil
\caption{The optimal solutions in each parking scenario. (a) Vertical scenario. (b) Parallel scenario. (c) Oblique scenario. These solutions depend on the underlying NLP implementations and the method for modelling collision constraints. The optimal trajectories are depicted in black solid lines. The vehicle is denoted as a rectangle in blue. } 
\label{fig2}
\end{figure*}

\begin{table*}[htbp]
  \centering
  \caption{The obtained optimal results based on different collision formulations and initial guesses.}
  \begin{threeparttable}
    \begin{tabular}{cc|cc|llll|llll|cc}
    \toprule
    \multirow{2}[2]{*}{Methods} & \multirow{2}[2]{*}{Guess} & \multirow{2}[2]{*}{CN \tnote{1}}  & \multirow{2}[2]{*}{VN} & \multicolumn{4}{c|}{Vertical} & \multicolumn{4}{c|}{Parallel} & \multicolumn{2}{c}{Oblique \tnote{4}} \\
          &       &       &       & \multicolumn{1}{c}{$J$}  & \multicolumn{1}{c}{ctime(s)} &  \multicolumn{1}{c}{TS}   & $t_f$(s) & \multicolumn{1}{c}{$J$} & \multicolumn{1}{c}{ctime(s)}  & \multicolumn{1}{c}{TS}    & $t_f$(s) &  \multicolumn{1}{c}{ctime(s)} & \multicolumn{1}{c}{$J$} \\
    \midrule
    \multirow{3}[2]{*}{(\ref{eq14})} & simplified & \multirow{3}[2]{*}{3446} & \multirow{3}[2]{*}{1426} & 1.907 & 3.515  & 0.05  & 171.3 & \textcolor[rgb]{ 1,  0,  0}{1.704} \tnote{3} & \textcolor[rgb]{ 1,  0,  0}{13.016 } & \textcolor[rgb]{ 1,  0,  0}{0.064} & \textcolor[rgb]{ 1,  0,  0}{141.49} & 2.785  & \multirow{18}[12]{*}{0.856} \\
          & collision-free  &       &       & \textbf{1.012} & 1.470  & \textbf{0.036} & \textbf{80.3} & {\textbf{1.221}} & 3.409  & {\textbf{0.042}} & {\textbf{93.94}} & 1.869  &  \\
          & hybrid $\tx{A}^{\star}$ &       &       & 1.19  & 1.610  & 0.044 & 97.2  &   {\textbf{1.221}}    & 1.688  &  {\textbf{0.042}}  &  {\textbf{93.94}}  & 1.586  &  \\
\cmidrule{1-13}    \multirow{3}[2]{*}{(\ref{eq16})} & simplified & \multirow{3}[2]{*}{1286} & \multirow{3}[2]{*}{306} & 1.19  & 0.733  & 0.044 & 97.2  & \textcolor[rgb]{ 1,  0,  0}{1.683} & \textcolor[rgb]{ 1,  0,  0}{3.223 } & \textcolor[rgb]{ 1,  0,  0}{0.063} & \textcolor[rgb]{ 1,  0,  0}{138.67} & 1.131  &  \\
          & collision-free  &       &       & \textbf{1.012} & 0.725  & \textbf{0.036} & \textbf{80.3} & {\textbf{1.221}} & 0.834  & {\textbf{0.042}} & {\textbf{93.94}} & 0.827  &  \\
          & hybrid $\tx{A}^{\star}$ &       &       & 1.19  & 0.596  & 0.044 & 97.2  &   {\textbf{1.221}}    & 0.712  &   {\textbf{0.042}}    &  {\textbf{93.94}}     & 1.001  &  \\
\cmidrule{1-13}    \multirow{3}[2]{*}{(\ref{eq17})} & simplified & \multirow{3}[2]{*}{\textbf{1166}\tnote{2}}  & \multirow{3}[2]{*}{\textbf{266}} & {1.19} & 0.722  & {0.044} & {97.2} & \textcolor[rgb]{ 1,  0,  0}{1.647} & \textcolor[rgb]{ 1,  0,  0}{2.561 } & \textcolor[rgb]{ 1,  0,  0}{0.055} & \textcolor[rgb]{ 1,  0,  0}{136.26} & 0.981  &  \\
          & collision-free  &       &       &    {1.19}   & \textbf{0.510 } &    {0.044}   &   {97.2}    & {\textbf{1.221}} & \textbf{0.703 } & {\textbf{0.042}} & {\textbf{93.94}} & \textbf{0.708 } &  \\
          & hybrid $\tx{A}^{\star}$ &       &       & \textbf{1.012} & 0.538  & \textbf{0.036} & \textbf{80.3} &   {\textbf{1.221}}    & 0.920  &  {\textbf{0.042}}     &   {\textbf{93.94}}    & 0.964  &  \\
\cmidrule{1-13}    \multirow{3}[2]{*}{(\ref{eq18})} & simplified & \multirow{3}[2]{*}{3606} & \multirow{3}[2]{*}{1426} & {1.19} & 2.057  & {0.044} & {97.2} & \textcolor[rgb]{ 1,  0,  0}{1.686} & \textcolor[rgb]{ 1,  0,  0}{7.428} & \textcolor[rgb]{ 1,  0,  0}{0.058} & \textcolor[rgb]{ 1,  0,  0}{138.47} & 3.061  &  \\
          & collision-free  &       &       &    {1.19}   & 1.971  &   {0.044}    &   {97.2}    & \textbf{1.221} & 4.022  & \textbf{0.042} & \textbf{93.94} & 3.127  &  \\
          & hybrid $\tx{A}^{\star}$ &       &       & 1.907 & 3.472  & 0.05  & 171.3 & 1.344 & 2.883  & 0.047 & 103.81 & 2.669  &  \\
\cmidrule{1-13}    \multirow{3}[2]{*}{(\ref{eq19})} & simplified & \multirow{3}[2]{*}{3766} & \multirow{3}[2]{*}{1426} & \multirow{3}[2]{*}{1.19} & 2.047  & \multirow{3}[2]{*}{0.044} & \multirow{3}[2]{*}{97.2} & \textcolor[rgb]{ 1,  0,  0}{1.699} & \textcolor[rgb]{ 1,  0,  0}{8.216 } & \textcolor[rgb]{ 1,  0,  0}{0.059} & \textcolor[rgb]{ 1,  0,  0}{139.58} & 3.174  &  \\
          & collision-free  &       &       &       & 2.176  &       &       & {\textbf{1.221}} & 2.572  & {\textbf{0.042}} & {\textbf{93.94}} & 2.546  &  \\
          & hybrid $\tx{A}^{\star}$ &       &       &       & 1.486  &       &       &   {\textbf{1.221}}    & 2.539  &  {\textbf{0.042}} &{\textbf{93.94}}   & 3.544  &  \\
\cmidrule{1-13}    \multirow{3}[2]{*}{(\ref{eq21})} & simplified & \multirow{3}[2]{*}{3766} & \multirow{3}[2]{*}{1426} & \multirow{3}[2]{*}{1.19} & 2.485  & \multirow{3}[2]{*}{0.044} & \multirow{3}[2]{*}{97.2} & \textcolor[rgb]{ 1,  0,  0}{1.684} & \textcolor[rgb]{ 1,  0,  0}{10.137 } & \textcolor[rgb]{ 1,  0,  0}{0.063} & \textcolor[rgb]{ 1,  0,  0}{138.72} & 3.053  &  \\
          & collision-free  &       &       &       & 1.894  &       &       & {\textbf{1.221}} & 2.501  & {\textbf{0.042}} & {\textbf{93.94}} & 2.762  &  \\
          & hybrid $\tx{A}^{\star}$ &       &       &       & 1.182  &       &       &    {\textbf{1.221}}   & 3.374  & {\textbf{0.042}} &{\textbf{93.94}}    & 2.047  &  \\
    \bottomrule
    \end{tabular}%
    \begin{tablenotes}   
        \footnotesize    
        \item[1] CN and VN are the total numbers of NLP constraints and variables to be optimized by using different methods. The computation time of the solver is ctime. The parking duration is $t_f$. The value $ \tx{TS}=\sum_{{k}=0}^{{k}_f-1} \left(a(k)^2+\omega(k)^2\right)$ is calculated to evaluate the trajectory smoothness.
        \item[2] The bold values denote the least of each column.     
        \item[3] The red values represent some obtained local solutions based on a simplified model guess. 
        \item[4] In this scenario, all locally optimal solutions are exactly the same based on different methods and various initial guesses. The TS and $t_f$ are 0.032 and \SI{70.18}{s}, respectively.
      \end{tablenotes}
    \end{threeparttable}
  \label{tab2}%
\end{table*}%

\subsection{Comparisons of computational demand}
Another primary observation in Table~\ref{tab2} is that the proposed formulations (\ref{eq16}), (\ref{eq17}) converge to local optimums very fast. In each scenario, when based on a same initial guess, it takes the least solving time for the proposed (\ref{eq16}), (\ref{eq17}) to solve an NLP. Even in vertical and oblique scenarios, the proposed (\ref{eq16}) and (\ref{eq17}) based on a simplified model guess are able to converge faster than other formulations based on good initial guesses. In all scenarios, when inputting a simplified collision-free guess or hybrid $\tx{A}^{\star}$ guess, it only takes hundreds of ms to solve using the proposed methods (\ref{eq16}), (\ref{eq17}), while other methods need thousands of ms, showing a considerable reduction of computational demand. What's more, it can be seen that the proposed method (\ref{eq14}) exhibits the same computational demand with existing formulations (\ref{eq18}), (\ref{eq19}), and (\ref{eq21}). The lower computational demand for the proposed methods (\ref{eq16}) and (\ref{eq17}) is because they employ fewer variables needed to be optimized and fewer constraints compared with others. Formulating the proposed method (\ref{eq14}) needs the same number of auxiliary variables as the existing methods (\ref{eq18}), (\ref{eq19}), and (\ref{eq21}), so their speed of solving the NLP is about the same.

To summarize, the proposed methods (\ref{eq16}) and (\ref{eq17}) have strong convergence and are more efficient, indicating that we can obtain a better local solution with less solving time in each parking scenario, and the proposed method (\ref{eq14}) performs equally well with state-of-the art methods.

\begin{remark}
The idea behind the proposed formulation (\ref{eq14}) is that a vehicle does not collide with obstacles when their vertices are kept outside each other. However, one possible issue of this kind of formulation is when the autonomous vehicle encounters narrow polygonal obstacles, the vehicle might opt to traverse through the obstacle in one sample, so a situation may occur that although the resulting trajectory is not collision-free, it still meets the collision avoidance constraints in each sample ($k=0,1,\cdots,k_f$). This issue can be circumvented by setting a shorter sampling interval.
\end{remark}

\begin{remark}
In this paper, the implementation in Matlab extensively uses the routines in CasADi and its harness for the solver IPOPT. It is quite possible to achieve improvements in computation time by exploiting other interior point solvers and different programming languages, so these results should be seen as a proof-of-concept rather than as a case of benchmarking a mature implementation of the algorithms. Besides, due to the computational performance's dependence on the underlying NLP solver and its implementation, an active-set sequential quadratic programming (SQP) solver might exhibit significant variations compared to an interior point solver \cite{ref21}. In our investigation, we attempted to solve the NLP (\ref{eq22}) using the active-set SQP algorithm implemented in Knitro \cite{ref39}. However, the computational cost of employing Knitro to solve the NLP (\ref{eq22}) proved to be considerably high, e.g., it takes 377s for CPU to converge to J=0.856 by using (\ref{eq18}) based on a hybrid $\tx{A}^{\star}$ guess of oblique parking. In some cases, the solver could not provide an optimal solution even with favorable initial guesses. Maybe other commercial active-set solvers might assist in evaluating the computational demands, but their proprietary nature prevents us from exploring these options. Nevertheless, the manipulation of auxiliary variables in the proposed (\ref{eq16}) and (\ref{eq17}) remains relatively minimal. Thus, we anticipate that the computational demands associated with active-set solvers should also be lower when compared to existing methods (\ref{eq18}), (\ref{eq19}), and (\ref{eq21}), as well as the proposed (\ref{eq14}). However, the computational comparisons among (\ref{eq14}), (\ref{eq18}), (\ref{eq19}), and (\ref{eq21}) are left to be investigated. 
\end{remark}

\section{Conclusion}\label{sec8}
This paper proposes three exact and explicit formulations of collision avoidance constraints. Through evaluations in the context of autonomous parking scenarios, we compare these proposed formulations with state of the art. The results highlight all formulations' ability to accurately model collision constraints, making them suitable for assisting autonomous vehicles in determining optimal trajectories even within confined surroundings. Notably, the formulations utilizing the hyperplane separation theorem significantly reduce an NLP problem size and show notable computational efficiency. We conclude that for robust and efficient optimization-based trajectory planning in autonomous parking scenarios, the combination of appropriate warm starting with formulations (\ref{eq16}) or (\ref{eq17}) is highly effective, e.g., a hybrid $\tx{A}^{\star}$ guess and formulation (\ref{eq17}). Moreover, the low computational demand when using (\ref{eq16}) and (\ref{eq17}) is valuable for real-time avoidance of static and dynamic obstacles in complex traffic scenarios. 

Future work may focus on developing better methods for warm starting the NLP. Additionally, the investigation of online trajectory planning maneuvers, leveraging the proposed efficient methods, is worth further exploring.

\end{document}